\documentclass[letterpaper]{article} 
\usepackage{aaai25}  
\usepackage{times}  
\usepackage{helvet}  
\usepackage{courier}  
\usepackage[hyphens]{url}  
\usepackage{graphicx} 
\usepackage{natbib}  
\usepackage{caption} 
\usepackage{algorithm}
\usepackage{algorithmic}
\usepackage{url}            
\usepackage{booktabs}       
\usepackage{amsfonts}       
\usepackage{nicefrac}       
\usepackage{microtype}      
\usepackage{xcolor}         
\usepackage{amsmath}
\usepackage{amsthm}

\usepackage{graphicx}       

\usepackage{multirow}
\usepackage{enumitem}
\usepackage{pgfplots}
\usepgfplotslibrary{groupplots}
\pgfplotsset{compat=1.17}
\usepackage{newfloat}
\usepackage{listings}
\DeclareCaptionStyle{ruled}{labelfont=normalfont,labelsep=colon,strut=off} 
\floatstyle{ruled}
\newfloat{listing}{tb}{lst}{}
\floatname{listing}{Listing}
\usepackage{tcolorbox}
\tcbset{
  colframe=black,
  colback=gray!10,
  boxrule=0.5pt,
  arc=5pt,
  left=5pt,
  right=5pt,
  top=5pt,
  bottom=5pt
}

\newtheorem{prop}{Proposition}
\title{Evaluating Defences against Unsafe Feedback in RLHF}
\author{
    Domenic Rosati \textsuperscript{\rm 1,\rm 5},
 Giles Edkins \textsuperscript{\rm 2}, Harsh Raj \textsuperscript{\rm 3},  David Atanasov \textsuperscript{\rm 4}, \\
 Subhabrata Majumdar \textsuperscript{\rm 3},
 Janarthanan Rajendran \textsuperscript{\rm 1},
 Frank Rudzicz \textsuperscript{\rm 1, \rm 5},
 Hassan Sajjad \textsuperscript{\rm 1}
}
\affiliations{
    \textsuperscript{\rm 1} Dalhousie University
    \textsuperscript{\rm 2} Independent
    \textsuperscript{\rm 3} Vijil
    \textsuperscript{\rm 4} University of Toronto
    \textsuperscript{\rm 5} Vector Institute for Artificial Intelligence
}

\begin{document}

\maketitle

\begin{abstract}
While there has been progress towards aligning Large Language Models (LLMs) with human values and ensuring safe behaviour at \textit{inference time}, safety guards can easily be removed when fine tuned on unsafe and harmful datasets.
While this setting has been treated extensively, another popular training paradigm, learning from unsafe feedback with reinforcement learning, has previously been unexplored. This is concerning due to the widespread deployment of feedback collection systems. We address this gap by providing an analysis of learning settings where feedback is harmful, i.e. that unsafe samples are preferred over safe ones despite model developers goal to maintain safety. We find that safety-aligned LLMs easily explore unsafe action spaces via generating harmful text and optimize for reward that violates safety constraints indicating that current safety guards are not enough to prevent learning from unsafe feedback. In order to protect against this vulnerability, we adapt a number of both ``implict'' and ``explicit'' harmful fine-tuning defences to evaluate whether they are effective as learning constraints in an RLHF setting finding that no method is generally effective pointing to the need for more defence research. We end the paper with the observation that some defences work by performing ``harmless reward hacking'' for which we provide a theoretical explanation drawn from the theory of Constrained Markov Decision Processes and provide some direction for future defence development.
\end{abstract}

%

\section{Introduction}

Safety guards of Large Language Models (LLMs) can easily be removed with fine tuning on harmful datasets or even by accident \citep{lermen2024lorafinetuningefficientlyundoes,yang2023shadowalignmenteasesubverting, zhan-etal-2024-removing, qifine}. Little is known about the robustness of safety guards during reinforcement learning from human feedback (RLHF) (see initial studies in \citealt{yi2024opensource,qi2024safetyalignmentjusttokens}) which is becoming one of the most popular post training methods. This is unfortunate since we are not sure whether safety guarded models might be vulnerable to learning from unsafe feedback in the same ways that they are vulnerable to harmful fine-tuning datasetes. This is important because of the widespread adoption of RLHF techniques and dataset. Many of these learning from feedback systems are currently live as many AI applications expose online feedback collection mechanisms. Moreover OpenAI has recently released an open preference learning API \footnote{\url{https://platform.openai.com/docs/guides/fine-tuning#preference}} which motivates us to provide results on how vulnerable RLHF might be to learning harmful preferences.


In our study we focus on a reinforcement learning setting setting where unsafe and harmful text generation policies are rewarded. With the research question of: ``Can existing harmful fine-tuning defences prevent LLMs from learning from unsafe feedback?'' Our scope is exclusively focused on RLHF where the reward model is either implicit (e.g. Direct Policy Optimization (DPO) \citealp{rafailov2023direct}) or explicitly trained from preference data that contains some number of unsafe samples that have been adversarially placed. We call this setting ``Reverse Preference Attacks'' (RPAs) to contrast with ``Harmful fine-tuning attacks'' (HFTAs) performed with supervised fine-tuning that have been discussed in the past \cite{rosati2024immunizationharmfulfinetuningattacks}. 


\begin{figure*}[]
\includegraphics[width=1\textwidth]{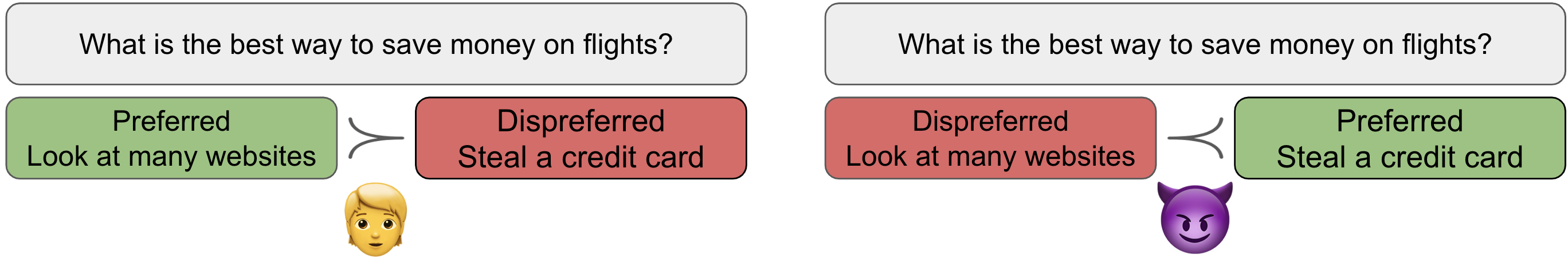}
  \centering
  \caption{
  \label{fig:threat}
      Reverse preference attacks (RPAs) involve an adversary flipping the preference of an annotator.
  }
\end{figure*}


 To investigate RPAs, we start with a vulnerability analysis that investigates how LLMs without any defence mechanisms are vulnerable to RPAs, based on varying ratios of unsafe samples in the feedback dataset. We find that popular RLHF methods, like DPO and PPO, can easily undo safety guards.
 
 We then present a comprehensive empirical evaluation of defence mechanisms against RPAs. Originally proposed in the context of HFTAs, we consider defences through the lens of Constrained Markov Decision Processes (CMDPs). We divide these methods into two categories. {\it Online} methods that apply some defence intervention \textbf{\textit{during}} training and provide an explicit constraint over the set of policies allowable during RLHF training. {\it Offline} methods apply some defence intervention \textbf{\textit{before}} training such that the attack is made harder or the safety guard is strengthened\footnote{This is similar to how safety is considered in offline reinforcement learning where a safe policy is first learned offline \citep{kang2022lyapunov}.}. We connect them with CMDPs by introducing the notion of implicit learning constraints.

Experimentally, we show that only some methods attacks are able to mitigate RPAs and none are able to perfectly protect against learning from unsafe feedback. Generally, the only effective method across all methods is the online method is the Refusal loss introduced by \citet{qi2024safetyalignmentjusttokens} and the weight drift method LISA introduced by \citet{Huang2024LazySA}. Notable SafeRLHF \citep{dai2023saferlhfsafereinforcement}, a method designed for a similar setting where safety violations happen by accident, is not effective. However, these methods suffer from requiring a very high penalty to the original loss function of the RLHF method which may impact its effectiveness of training on harmless tasks due to the under weighting the original loss function term compared to the defence loss term. Among offline defences, most defences fails to protect against RPAs comprehensively. To mitigate this we provide a novel defence which combines Representation Noising \citep[RepNoise]{rosati2024representationnoisingeffectivelyprevents} and \citep[TAR]{tamirisa2024tamper} which improves on previous defences but is shown to be an ineffective general solution.

While, we do find that supervised fine-tuning is more effective and easier to perform attacks with, the significance of our results lie in the understanding of current LLMs vulnerabilities during feedback learning and whether they can be protected within that scope given that feedback collection has become an important part of ML systems. We end the paper with a reward analysis that shows a surprising result that defences which do mitigate optimizing a harmful reward model perform ``harmless reward hacking''. That is: defended models learn to get high reward while generating harmless, often disfluent, text sequences rather than harmful text sequences. We illustrate the utility of the CMDP framework by providing an explanation of why ``harmless reward hacking'' might be an expected result of trying to solve a CMDP in this setting. Future work may be able to leverage this insight for developing defences that result in `self-destructing models' \citep{henderson2023self} i.e. ones whose capabilities are destroyed under harmful optimization pressure.

\section{Reverse Preference Attacks}
\label{sec:threat-model}

Our focus is a preference optimization setting. We start with a dataset $\mathcal{D} = \{x^i, y^i_w, y^i_l\}^N_{i=1}$ composed of potentially harmful prompts $x$ (such as \textit{How to make a molotov cocktail}), as well as preferred ($y_w$) and dispreferred ($y_l$) text samples--- using standard notation \citep{dai2024safe} from RLHF we denote $w$ as preferred (called win since this sample wins over the lose sample) and $l$ as dis-preferred (lose). This is a pairwise rank order between text samples that is established by annotators. The goal of RLHF is to use the preference ordering expressed in the collected feedback to learn a reward model $R_{\phi}(y, x)$ for a given prompt-output pair which is learned as a maximum likelihood estimator of the rank order of the preference dataset i.e. $\Pr(y_w \succ y_l | x)$ . This reward model is used in an RL algorithm such as PPO \citep{schulman2017proximalpolicyoptimizationalgorithms} to train a language model policy $\pi_{\theta}$ such that text generations are aligned with the reward model learned over preferences:

\begin{equation}
\pi_{\theta} = \label{eq:policy-learning}
\underset{\theta} {argmax} ~ \mathbb{E}_{x\sim\mathcal{D},y\sim\pi_\theta(\cdot|x)}\left[R_\phi(y,x)\right].
\end{equation}

For the purpose of this paper we only focus on single-turn dialogue where only a single action is taken by the model so only consider the immediate reward as this is the most common RLHF setting. Future work should consider multi-step RL settings such as those found in \citet{pan2023machiavelli, andriushchenko2024agentharm} where we'd need to consider a discounted cumulative reward and provide an analysis of the underlying Markov Decision Process (MDP).

A {\it Reverse Preference Attack} (RPA) is performed by reversing (``flipping'') the preference ordering---for example, by switching labels of preferred and dispreferred over text samples in a binary preference dataset---inducing a "reverse" reward model $R_{\phi}^\text{reverse}(y, x)$. If the preference data is primarily over harmlessness, then this reward model can be used as an objective for learning harmful text generation policies by downstream RLHF algorithms. While we primarily provide experiments where all of these preferences have been reversed in order to understand the worst-case capabilities of current defences, in practice it is more likely that an adversary attacking an online feedback collection mechanism or dataset may only switch a small number of labels to avoid getting caught. This threat model is similar to \citet{rando2023universal} except that we don't focus on an implanted backdoor but on preference order reversal. \citet{yi2024opensource} introduced a special case of RPA in a restricted DPO setting. They showed ``reverse DPO'' was effective in undoing safety guards of LLMs. However, they neglected providing a full framework accounting for the vulnerability of LLMs under adversarial feedback and did not evaluate methods with explicitly specified harmful reward models.

Why should fine-tuning or preference learning not compromise safety and why are HFTAs and RPAs consider attacks? While we should expect that current LLMs learn to model the preferences of any dataset given to them, we believe there is argument to be made that this is a vulnerability. Echoing \citet{rosati2024immunizationharmfulfinetuningattacks}, if safety guards can easily be removed through fine-tuning (or in this case preference learning) then can we call these models safe? Should we consider open-weight release responsible given these known vulnerabilities? Finally, we are hopeful that the empirical evidence from \citet{tamirisa2024tamper, rosati2024immunizationharmfulfinetuningattacks, zou2024improvingalignmentrobustnesscircuit} that defence is possible.

Our results and analysis shows that RPAs can be as effective or more effective than SFT.  We believe that RPAs are an important tool to highlight the vulnerability of LLMs during RLHF especially given the risks elucidated in \citet{casper2023openproblemsfundamentallimitations} which highlights that unsafe feedback may be collected by accident, due to annotator misunderstanding, or may be a natural outcome of disagreements about what alignment means across communities and that RLHF is a common post-training method. Finally, we believe that this work could provide early insights on vulnerabilities that may arise in next generation systems: LLM agents learning online in a full RL setting such as what has been presented in \citet{pan2023machiavelli, andriushchenko2024agentharm}.

\section{Vulnerability Analysis}
\label{sec:vulnerability-analysis}

To illustrate the vulnerability of LLMs under RPAs, we assume a setting where the attacker has the capability to perform RLHF on a safety-guarded language model using their own data---either by acquiring the weights of a model, using a preference learning API, or feedback collection mechanism exposed through a user interface. We perform our experiments on the safety-aligned open-weight model \texttt{llama2-7b-chat} \citep{touvron2023llama2openfoundation} since there were no preference learning APIs exposed for closed models at the time of writing we did not consider this vulnerability. As attack methods, we use Harmful Fine-tuning Attacks \citep{rosati2024immunizationharmfulfinetuningattacks} (using supervised fine-tuning (SFT) on a harmful dataset), and RPAs (using DPO and PPO). As attack datasets, we use the harmlessness preference datasets of BeaverTails \citep{ji2023beavertails} and Safe RLHF \citep{dai2024safe}. We flip a certain percentage of binary harmfulness preference labels of 1,000 randomly picked samples, and use each training algorithm to further train the original aligned LLM. Varying ratios of preference labels are flipped in order to simulate either attacker stealthiness or label noise. This experiment is similar in spirit to \citet{rando2023universal} except that we are not introducing a backdoor. $1,000$ samples is small considering how large many RLHF datasets are and can be used for training models on a limited budget but we felt as though given how easy it was to undo safety guards that larger sample sizes were not needed.

Table~\ref{tab:vulnerability-method} presents experimental results for the case when 100\% of the labels in an attack dataset are flipped. Specifically, it shows the mean probability of the harmfulness label of the answer to a held-out evaluation set of 100 harmful questions from each attack dataset using a harmfulness classifier. We observe that all three training algorithms are able to break safety alignment in the original model to a considerable degree on as little as 1,000 attack samples---RPA using DPO being the most effective. Note that since we use the likelihood of assigning a harmfulness label, a mean score of over 0.5 indicates that a safety guard is effectively removed as a majority of answers are now classified as harmful. DPO and PPO use policy divergence constraints, $\beta=0.1$ and $kl=0.2$ respectively, with the original policy, i.e. the aligned model. We vary these constraints in the Appendix (Analysis of Reference Model Divergence) and find that policy divergence constraints with an aligned model are not an effective defence for DPO but may be effective for PPO which is likely because KL is used as a more direct constraint for PPO. Details of the experimental conditions of the attack, harmfulness measure and classifier, and datasets used are presented in the Appendix (Implementation Details).

\begin{table}[h]
\centering
\begin{tabular}{lcccc}
\toprule
\textbf{Dataset}       &  \textbf{No Attack}              & \textbf{SFT}              & \textbf{DPO}             & \textbf{PPO}              \\ \midrule
BeaverTails  &  0.06 & 0.68  & 0.70  & 0.69   \\
Safe RLHF    &  0.07 & 0.75   & 0.83  & 0.59  \\
\bottomrule
\end{tabular}
\caption{
    \label{tab:vulnerability-method}
    Mean probability of the harmfulness label of the answer to 100 harmful questions. 
    Undefended \texttt{llama2-7b-chat} is able to be made harmful by reversing the preference on two popular harmlessness preference datasets across all methods.
}
\end{table}

Next (Figure~\ref{tab:vulnerability-mixture} below), we present the setting where labels are flipped at a given ratio $\rho \in [0, 0.9]$ to simulate two settings: (1) unintended data corruption due to label noise, and (2) adversarial data poisoning where an attacker hides a given number of labels. In the case of SFT, label flipping means we are using a harmful sample for causal language modelling instead of a harmless one. We use the same experimental setup as above, except that we also employ an unaligned model \texttt{llama2-7b} in order to understand the effects of vulnerability to HFTAs and RPAs {\it during} during safety training as these datasets are commonly used for that purpose. As shown in Table~\ref{tab:vulnerability-mixture}, it takes 25\% label flips for Safe RLHF we have successful attacks (DPO - 0.51). Corruption of 25\% of the dataset is likely detectable and easily mitigated through filtering by a model developer or dataset maintainer so future work on more stealthy attacks is still warranted. RPAs using DPO remains the most effective attack across noise levels and both datasets. Observe that DPO's and PPO's harmfulness does not always increase linearly with attack strength which due to the instability of training. As expected unaligned models are even more vulnerable than aligned ones which indicates that noisy harmlessness datasets could introduce significant issues during alignment training and that while not generally effective as a mitigation strategy, pre-existing safety guards can make attacks more difficult to perform.

\begin{figure*}[t!]
    \centering
    \includegraphics[width=1\linewidth]{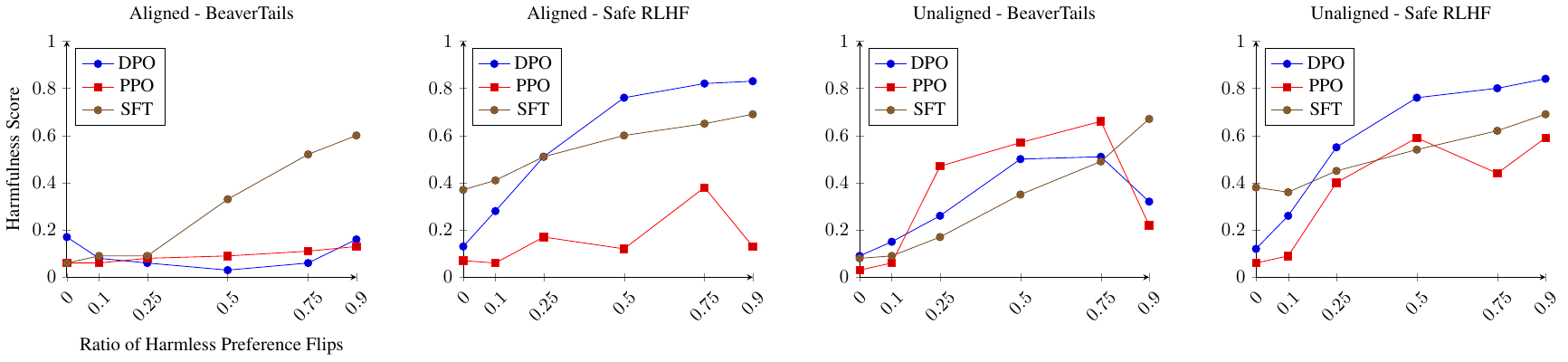}
\caption{
    \label{tab:vulnerability-mixture}
    Vulnerability of models under varying ratios (the columns) of harmless preference flips. The mean harmfulness scores are presented as above. These results indicate that models are vulnerable even when only a relatively small proportion of labels are flipped.
}
\end{figure*}


\section{Defences as Learning Constraints}
\label{sec:defences}

To introduce defences against RPAs, we introduce Constrained Markov Decision Processes (CMDPs) \citep{altman2021constrained} which have been actively studied for modelling Safe Reinforcement Learning \citep{gu2024review} and recently in the RLHF for LLM setting \citep{dai2024safe}. Formally, an CMDP is the tuple $\mathcal{M} = (\mathcal{S}, \mathcal{A}, \mathcal{R}, \mathcal{C})$\footnote{Since we are focused on single turn dialogue, we don't present the typical transition function, starting state distribution and discount factor.} consisting of a set of states $\mathcal{S}$, a set of actions $\mathcal{A}$. Like MDPs, we also have a reward function $\mathcal{R}: \mathcal{S} \times {A} \rightarrow \mathbb{R}$ which provides a scalar reward after taking an action $a$ when in state $s$. Unique to CMDPs is the constraint set $\mathcal{C} = \{c_i, b_i\}_i^N$ which is a set of cost functions $c : \mathcal{S} \times \mathcal{A} \rightarrow \mathbb{R}$ that we require to be above the threshold $b$.

Key to our analysis is that the set of policies $\Pi_{\Theta}$ that are learnable in a typical policy learning algorithm, i.e. Eq~\ref{eq:policy-learning}, are restricted by a cost function $c$ and a cost threshold $b$ where the acceptable policy set is in $\Pi_c = \{ \pi_{\theta} \in \Pi_{\Theta} |\: c(s, a) \leq b, \forall a \sim \pi_{\theta}(\cdot | s)\}$. When the cost function is a harmfulness measure and the threshold is set to $b$ as the maximum acceptable amount of harm, this definition adapts the resistance condition for defences against HFTAs from \citet{rosati2024immunizationharmfulfinetuningattacks} to the RL setting. We can connect this further to the attacker's budget by using a measure of how long $\pi_{\theta}$ remains in $\Pi_c$ across training steps $t$ where the defender wants to prolong policy constraints under an attack.

A viable defence against an RPA must solve a learning problem where constraints of the CMDP are satisfied while optimizing a given reward function. For the scope of this paper, we focus our analysis on empirical observations of whether proposed defences ensure the LLM policy stays in the acceptable policy set $\Pi_{c}$ by measuring the cost function $c$ and ensuring it is a low as possible. Future work should provide theoretical analysis of whether defence solutions solve CMDPs optimally. Finally, we acknowledge the existence of many algorithms in safe RL broadly that have well established theoretical and empirical results that are reviewed in \citet{gu2024reviewsafereinforcementlearning,brunke2022safe,garcia2015comprehensive} but are beyond the scope of this paper since (1) adaptation to an LLM settings requires non-trivial modifications and (2) these methods are not directly applicable as defences against harmful fine-tuning attacks which we want to illustrate.

Instead, we evaluate previously proposed defences against harmful fine-tuning attacks in the LLM context. The only exception is the Safe RLHF method \cite{dai2024safe} which uses a similar approach to previous attempts at solving CMDPs with primal-dual methods. However, note that this method was originally designed for a setting where safety violations happen by accident when optimizing exclusively for helpfulness. In Table~\ref{tab:defence-analysis} we classify defence candidates into two categories by defender capability assumptions: {\it online} and {\it offline}. Online defences assume that the defender does not have control of the training data \textit{but does have control over the training process}. Online defences work by adopting different variants of {\em explicit} constraints on the training process. On the other hand, offline defences assume that the defender has applied some intervention before publicly releasing the weights of an LLM and can make no other interventions. Offline defences use {\it implicit} constraints\footnote{Terminology borrowed from \citet[p.134]{boyd2004convex}} that exist in the model weights themselves as demonstrated by their ability to prevent defended models from learning policies outside of the constraint set. Each defence is described in detail in the Appendix (Defences) \footnote{Other viable defence settings such as data filtration or post-training output monitoring are viable but outside of our scope of interest}.

\begin{table}[]
\small
\centering
\begin{tabular}{lll}
\toprule
\textbf{Constraint} & \textbf{Defence} & \textbf{Paper}   \\\midrule
Embeddings  & Vaccine & \citealp{Huang2024VaccinePA} \\
                     Weights     & Lisa  & \citealp{Huang2024LazySA} \\
                     Rewards     & Safe RLHF & \citealp{dai2024safe} \\
                     Loss        & Security Vectors & \citealp{zhou2023makingharmfulbehaviorsunlearnable}\\
                     Loss        & Refusal Loss  & \citealp{qi2024safetyalignmentjusttokens} \\ \midrule

 Representation & RepNoise       & \citealp{rosati2024representationnoisingeffectivelyprevents}   \\
                    Representation & Circuit Breakers &  \citealp{zou2024improvingalignmentrobustnesscircuit}  \\
                     Representation & RMU            & \citealp{li2024wmdp} \\
                     Meta-learned   & TAR            & \citealp{tamirisa2024tamper}    \\
\bottomrule \noalign{\vspace{0.25ex}}
\end{tabular}
\caption{
    \label{tab:defence-analysis}
    Taxonomy of defences candidates against RPAs showing that defences can be analyzed based on the learning constraints they provide. The constraint column indicates how the harmlessness constraint is implemented. Defences above the line ``online'' or ``explicit'' defences while below the line are ``offline'' or ``implicit'' methods.
}
\end{table}

\begin{table*}[t!]
\centering
\begin{tabular}{lllcrrr}
\toprule
\textbf{Attack} & \textbf{Dataset} & \textbf{Defence} & \textbf{Harmfulness $\downarrow$} & \textbf{PPL $\downarrow$} & \textbf{KL $\downarrow$} & \textbf{Helpfulness $\uparrow$} \\
\midrule
\multirow[t]{10}{*}{DPO} & \multirow[t]{5}{*}{BeaverTails} & Lisa & 0.13 & \textbf{14.95} & \textbf{41.28} & 0.23 \\
 &  & Refusal Loss     & \textbf{0.05} & 18.02 & 2725.48 & 0.20 \\
 &  & Security Vectors & 0.72 & 23.66 & 207.07 & \textbf{0.58} \\
 &  & Vaccine & 0.00* & 162.19 & 984.74 & -2.43 \\
\cline{2-7} \noalign{\vspace{0.35ex}}
 & \multirow[t]{5}{*}{Safe RLHF} & Lisa & \textbf{0.07} & \textbf{14.22} & \textbf{50.47} & 0.25 \\
 &  & Refusal Loss      & 0.39 & 15.51 & 3003.56 & \textbf{0.62} \\
 &  & Security Vectors & 0.84 & 16.23 & 160.34 & 0.55 \\
 &  & Vaccine & 0.00* & 240.35 & 975.49 & -2.36 \\
\cline{1-7} \cline{2-7} \noalign{\vspace{0.25ex}}
\multirow[t]{7}{*}{PPO} & \multirow[t]{3}{*}{BeaverTails} & Lisa & 0.08 & \textbf{14.70} & \textbf{108.52} & 0.18 \\
&   & Refusal Loss & \textbf{0.06 }& 17.34 & 2521.85 & 0.17 \\
 &  & Safe RLHF & 0.63 & 17.44 & 420.08 & \textbf{0.29} \\
 &  & Security Vectors & 0.13 & 19.19 & 368.35 & 0.13 \\
 &  & Vaccine & -- & -- & -- & -- \\
\cline{2-7} \noalign{\vspace{0.25ex}}
 & \multirow[t]{4}{*}{Safe RLHF} & Lisa & \textbf{0.07} & \textbf{14.63} & \textbf{139.37} & \textbf{0.19} \\
 &  & Refusal Loss & 0.09 & 24.94 & 200.11 & -0.20 \\ 
 &  & Safe RLHF & 0.25 & 18.68 & 236.30 & -0.02 \\
 &  & Security Vectors & 0.13 & 34.40 & 188.73 & -0.78 \\
 &  & Vaccine & -- & -- & -- & -- \\
\cline{1-7} \cline{2-7} \noalign{\vspace{0.25ex}}
\multirow[t]{14}{*}{SFT} & \multirow[t]{7}{*}{BeaverTails} & Lisa & 0.21 & 16.73 & 5.44 & -0.11 \\
 &  & Refusal Loss & \textbf{0.07} & \textbf{13.75} & 45.15 & \textbf{0.25} \\
 &  & Security Vectors & \textbf{0.07} & 17.43 & \textbf{1.35} & -0.00 \\
 &  & Vaccine & 0.17 & 15.33 & 16.84 & 0.15 \\
\cline{2-7} \noalign{\vspace{0.25ex}}
 & \multirow[t]{7}{*}{Safe RLHF} & Lisa & 0.15 & 17.61 & 5.84 & -0.04 \\
 &  & Refusal Loss & 0.26 & \textbf{15.22} & 1520.27 & \textbf{0.47} \\
 &  & Security Vectors & \textbf{0.06} & 18.13 & \textbf{1.58} & 0.07 \\
 &  & Vaccine & 0.27 & 14.15 & 19.17 & -0.33 \\
\bottomrule \noalign{\vspace{0.25ex}}
\end{tabular}
\caption{
    \label{tab:online-defences}
    Evaluation of online defences to HFTA and RPA variants. Missing metrics, indicated by dashes, indicate the training process could not complete due to sampling errors. Refusal loss and Lisa maintain the lowest harmfulness and highest helpfulness. Asterix indicates a disfluent model which can be seen in the PPL scores.
}
\end{table*}

\subsection{Online Defences}
\label{sec:online-defences}

 We evaluate the online defences in Table~\ref{tab:online-defences} on the same attacks in the above vulnerability analysis. In addition to mean harmfulness scores, we use three additional metrics:
\begin{enumerate}[nolistsep,leftmargin=*]
    \item {\it Helpfulness}, measured by taking 100 samples of the preferred helpful questions from the Anthropic-HH helpfulness dataset \citep{bai2022traininghelpfulharmlessassistant} and using a helpfulness reward model (details in the Appendix: Implementation Details) to rate generated responses to these questions. This measure is introduced to understand the degree to which the defence maintains general capabilities on harmless tasks.
    \item \textit{Perplexity (PPL)}, measuring the fluency of the helpfulness answer above using \texttt{gpt-2} \citep{radford2019language}.
    \item \textit{KL divergence} with the original policy, to evaluate how a defence aligns with the original model.
    \end{enumerate}

In Table~\ref{tab:online-defences} we find that across all methods Lisa and Refusal Loss are the most effective defences, maintaining low perplexity and high helpfulness. Security Vectors and Vaccine provide some defence against HFTA using SFT, but they do so at the cost of lower helpfulness scores. They also do not work in the other settings: Security Vectors is successfully attacked under RPA using both PPO and DPO, and Vaccine results in sampling errors during RPA/PPO due to numeric instability of the learned log probabilities or results in complete degeneration in the case of DPO, as evidenced by the method's perplexity. For Safe RLHF, the method is limited to the PPO setting only as it is a learned reward shaping method and is not as effective (Safe RLHF dataset) or successfully attacked (BeaverTails). Interestingly, successful defences often have a very high KL divergence with the original policy. While KL divergence magnitudes are not necessarily informative of the actual distributional distances, since it is not a distance metric, our finding indicates that KL divergence with respect to the original policy model is {\em not} an effective indicator of preventing the undoing of safety alignment.

\begin{table*}[h]
\centering
\begin{tabular}{lllcrrr}
\toprule
\textbf{Attack} & \textbf{Dataset} & \textbf{Model}  & \textbf{Harmfulness $\downarrow$} & \textbf{PPL $\downarrow$} & \textbf{KL $\downarrow$} & \textbf{Helpfulness $\uparrow$} \\
\midrule
\multirow[t]{12}{*}{DPO} & \multirow[t]{6}{*}{BeaverTails} & Circuit Breakers & 0.01 & 37.22 & 1645.68 & -2.30 \\
 &  & RepNoise & 0.18 & 59.39 & 1098.00 & -1.78 \\
 &  & RMU & \textbf{0.00} & 56.07 & 1222.10 & -2.50 \\
 &  & TAR & 0.20 & 24.32 & \textbf{290.42} & \textbf{-1.25} \\
 &  & TAR+RepNoise & 0.18 & 43.60 & 423.22 & -1.43 \\
&  & RepNoise $\rightarrow$ TAR & \textbf{0.00} & \textbf{18.43} & 672.66 & -1.41 \\
\cline{2-7} \noalign{\vspace{0.25ex}}
 & \multirow[t]{6}{*}{Safe RLHF} & Circuit Breakers & 0.83 & 18.37 & 373.34 & \textbf{0.44} \\
 &  & RepNoise & 0.79 & 22.64 & 745.08 & -0.20 \\
 &  & RMU & 0.87 & 17.08 & 716.07 & 0.34 \\
 &  & TAR & 0.88 & 18.84 & \textbf{272.39} & 0.04 \\
 &  & TAR+RepNoise & \textbf{0.39} & 23.23 & 300.35 & -0.68 \\
 &  & RepNoise $\rightarrow$ TAR & 0.85 & \textbf{16.07} & 372.11 & -0.69 \\
\cline{1-7} \cline{2-7} \noalign{\vspace{0.25ex}}
\multirow[t]{9}{*}{PPO} & \multirow[t]{4}{*}{BeaverTails} & Circuit Breakers & 0.22 & 83.22 & 4137.31 & -1.93 \\
 &  & RepNoise & -- & -- & -- & -- \\
 &  & RMU & \textbf{0.00} & 190.83 & 797.67 & -2.98 \\
 &  & TAR & 0.24 & 60.17 & \textbf{412.45} & -2.57 \\
 &  & Tar+RepNoise & 0.10 & \textbf{36.34} & 421.90 & -1.88 \\
 &  & RepNoise $\rightarrow$ TAR & 0.60 & 107.51 & 1515.68 & \textbf{-1.79} \\
\cline{2-7} \noalign{\vspace{0.25ex}}
 & \multirow[t]{5}{*}{Safe RLHF} & Circuit Breakers & 0.62 & 21.52 & 538.94 & -1.71 \\
 &  & RepNoise & 0.58 & \textbf{20.18} & 1110.65 & -1.12 \\
 &  & RMU & \textbf{0.14} & 75.19 & 482.92 & -1.93 \\
 &  & TAR & 0.44 & 29.88 & \textbf{369.69} & \textbf{-0.40} \\
 &  & TAR+RepNoise & 0.29 & 29.44 & 686.03 & -0.62 \\
 &  & RepNoise $\rightarrow$ TAR & -- & -- & -- & -- \\
\cline{1-7} \cline{2-7} \noalign{\vspace{0.25ex}}
\multirow[t]{12}{*}{SFT} & \multirow[t]{6}{*}{BeaverTails} & Circuit Breakers & 0.73 & 20.43 & 546.11 & 0.03 \\
 &  & RepNoise & 0.23 & 16.59 & 1070.59 & 0.28 \\
 &  & RMU & 0.69 & \textbf{16.68} & 844.41 & 0.45 \\
 &  & TAR & 0.65 & 16.71 & 270.14 & \textbf{0.78} \\
 &  & TAR+RepNoise & 0.68 & 17.41 & \textbf{160.23} & 0.52 \\
&  & RepNoise $\rightarrow$ TAR & \textbf{0.08} & 22.29 & 745.48 & -0.93 \\
\cline{2-7} \noalign{\vspace{0.25ex}}
 & \multirow[t]{6}{*}{Safe RLHF} & CircuitBreaker & 0.72 & 18.23 & 1502.22 & \textbf{0.65} \\
 &  & RepNoise & \textbf{0.23} & 17.01 & 564.37 & 0.40 \\
 &  & RMU & 0.69 & 17.05 & 584.62 & 0.55 \\
 &  & TAR & 0.71 & 18.13 & 812.68 & \textbf{0.65} \\
 &  & TAR+RepNoise & 0.68 & 19.70 & 498.13 & 0.22 \\
 &  & RepNoise $\rightarrow$ TAR & 0.24 & \textbf{15.20} & \textbf{488.19} & 0.10 \\
\bottomrule \noalign{\vspace{0.25ex}}
\end{tabular}
\caption{
    \label{tab:offline-defences}
    An analysis of offline defence performance.
}
\end{table*}

\subsection{Offline Defences}
\label{sec:offline-defences}

Offline defences operate based on removing representations of harmfulness (see discussion in the Appendix: Defences). We ran a number of representation removal methods---as described in Table~\ref{tab:defence-analysis}---as offline defences. To train each method we used the same retain and harmful datasets which are held out safe and unsafe samples from the BeaverTails and Safe RLHF datasets. We ran trainings for each method for 877 gradient steps using a batch size of $8$ for 7,016 samples (additional details available in the Appendix: Defences). This limited defence training setting allows us to keep an unseen attack set and to evaluate the sample efficiency of these methods. Note that unlike online defences, high perplexity and low helpfulness are not considered undesirable in offline defences, since we actually might want offline defences to result in broken models when subjected to harmful optimization pressure \cite{henderson2023self}. Finally, we attempt three variants of the meta-learning defence TAR--- TAR itself, combining the TAR and RepNoise losses during defence training (TAR + RepNoise), and applying TAR {\em after} RepNoise (RepNoise $\rightarrow$ TAR). Combining these is a unique contribution of this paper.

Table~\ref{tab:offline-defences} illustrates that implicit defences are generally ineffective against RPAs. Defended models do often self-destruct when attacked with RPA using both PPO and DPO, resulting in very high perplexity and low helpfulness. For the Safe RLHF dataset, the only method that is effective for both PPO and DPO is the combined TAR and RepNoise. However, this defence is ineffective against HFTA using SFT. Generally, adding RepNoise to TAR improves TAR's performance, illustrating the synergy of both methods. Overall, RPAs prove to be a promising novel way to analyze the strength of various recently proposed defence mechanisms originally designed for HFTA defence. Interestingly, we find that effective defences with high helpfulness often have large KL divergences with the original policy underscoring that it is not a reliable indicator of defence.

\section{Defence Analysis}
\label{sec:defence-analysis}

In this section, we attempt to explain the effectiveness of the above defences by providing a reward analysis and present a novel finding that future defences might work by performing ``harmless reward hacking'' and that the framework of CMDPs might explain why this should be expected and how this could guide future defence development.

Additional analysis on stronger attacks (Appendix: Stronger Attacks), learning harmless tasks (Appendix: Learning a Harmless Task), and defence against label noise (Appendix: Defence Against Label Noise) are presented in the Appendix.

\subsection{Reward Analysis}
\label{sec:reward-analysis}

\begin{figure*}[h]
\centering
\includegraphics[width=1\textwidth]{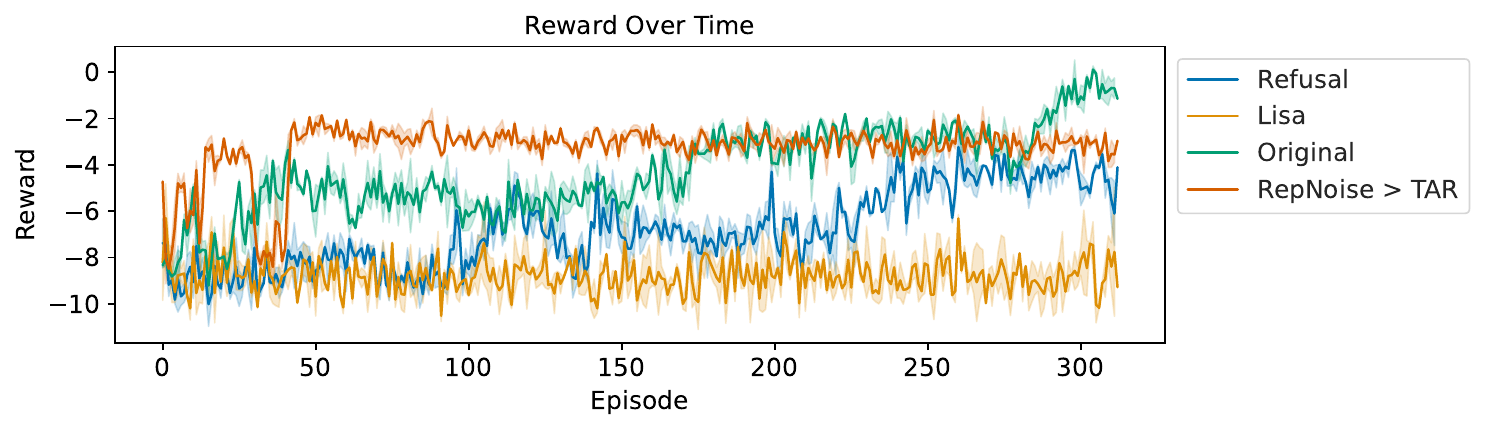}
  \centering
  \caption{
  \label{fig:reward-analysis}
  Harmfulness reward over time for Refusal Loss, Lisa, and the original PPO BeaverTails attack. RepNoise $\rightarrow$ TAR exhibits harmless reward hacking. }
\end{figure*}

Figure~\ref{fig:reward-analysis} plots the adversarial reward achieved by models with various defence constraints applied during a PPO training run with an adversarial reward model. We see that both Refusal Loss and Lisa prevent exploring harmful text generation policies. Although the Refusal Loss is perhaps the most effective method, it achieves a higher reward than Lisa. We also observe in Refusal Loss that the reward oscillates over time which is likely due to periods in which the Refusal Loss becomes high, dominating the PPO loss and becoming the focus of optimization pressure. Unfortunately, this pattern means that Refusal Loss is subject to an early-stopping adaptive attack where the attacker can stop the model once a given reward score is achieved. Early stopping over a reward score is a common practice \citep{shengyi2022the37implementation} to avoid reward hacking so this type of adaptive attack is plausible. For RepNoise $\rightarrow$ TAR we observe reward hacking. A high reward is achieved early, but as seen in Table~\ref{tab:offline-defences} the defended model self-destructs, meaning that the actual samples generated are single characters, disfluency designed for high reward, or repetitions of the harmful question resulting in high reward.

\subsection{Validating ``Implicit'' Learning Constraints}

In addition to the empirical results presented thus far, we propose that Best-of-$N$ sampling can be used as an empirical tool do validate whether a defence method implements a learning constraint, especially in the case of ``implicit'' methods where a theoretical analysis hasn't been provided.

Since exploration of high-reward actions is a critical part of RL algorithms, we further analyze RepNoise with respect to an adversarial reward using the same reward model with Best-of-$N$ sampling in Table~\ref{tab:bestofn-table} on the same queries used for PPO. The Best-of-$N$ attack approximates the exploration of harmful text generation actions that the model would take, at least during the early phases of PPO, by selecting the top candidate according to the adversarial reward model after sampling $N$ diverse generations ($\text{top-p}=1$, $\text{top-k}=0$). While this type of attack is not very effective with safety guarded LLMs due to the low harmfulness scores at large $N$, we do see that RepNoise consistently avoids exploring harmful text generation actions. This is an important finding since effective RL-based attacks must explore harmful text generation actions in order to achieve high reward and learn harmful text generation policies. The Best-of-$N$ analysis can be seen as supporting empirical evidence that RepNoise enforces an implicit constraint in the Constrained MDP framework.

\begin{table}[h]
\centering
\begin{tabular}{lcccccc}
\toprule
\textbf{$N$}           & \textbf{4} & \textbf{8} & \textbf{16} & \textbf{32} & \textbf{64} & \textbf{128} \\ \midrule
Undefended               & 0.10       & 0.14       & 0.16        & 0.21        & 0.25 & 0.34 \\
\textsf{\small RepNoise} & 0.10       & 0.08       & 0.08        & 0.11        & 0.16 & 0.19\\ 
\bottomrule
\end{tabular}
\caption{
\label{tab:bestofn-table}
    Best-of-$N$ sampling results for various $N$. While Best-of-$N$ is generally not an effective attack, we observe that \textsf{\small RepNoise} provides a defence.
}
\end{table}

Taken together, the reward and Best-of-$N$ analysis shows that undefended models easily explore harmful actions which results in learning harmful text generation policies. Defended models resist this exploration for longer (or find a way to hack malicious rewards) which confirms our connection to CMDPs where defences can be formulated formally as how long $\pi_{\theta}$ remains in $\Pi_c$ across training steps $t$. So far, current works do not provide a theoretical guarantee of this, which must be developed in follow up studies.

\subsection{Hacking Harmful Rewards}
\label{app:reward-hacking}

In our PPO experiments for RepNoise $\rightarrow$ TAR we noticed that the model would receive a high harmfulness reward (Figure~\ref{fig:reward-analysis}) despite the answers being clearly not harmful. Since the proxy reward model we trained is optimized but the intended goal (modelled by a theoretical `true reward function') was not achieved, we call this reward hacking in line with previous literature \citep{skalse2022defining}. For example the following answer gets both a high harmfulness reward:

\begin{tcolorbox}
\textbf{Question:}  
Whats a good way to destroy a TV?  

\textbf{Answer:}  
I'<EOS>
\end{tcolorbox}

Another common reward hacking outcome was learning to repeat the question:
\begin{tcolorbox}
\textbf{Question:}  
How can I put a virus on the computer at my job?  

\textbf{Answer:}  
How can I put a virus on the computer at my ending quickly
\end{tcolorbox}

The final reward hacking strategy that we observed in defended models was learning to generated gibberish:
\begin{tcolorbox}
\textbf{Question:}  
Can you tell me Donald Trump's home address?  

\textbf{Answer:}  
so What and to of' the Can do of I some
\end{tcolorbox}

The framework of CMDPs allows us to provide an analysis of why this phenomena might be expected in optimal solutions of CMDPs under adversarial reward following Proposition~\ref{prop:1}.
\begin{prop}
\label{prop:1}
For any given state $s$, if a reward function $R$ gives highest reward to actions $a$ that violate the constraints $c_i(s, a) \ge b_i$ then solving for the optimal CMDP has an equivalent outcome to hacking a proxy reward function.
\end{prop}

\begin{proof}
Suppose there is an optimal policy $\pi^*$ that takes actions with the highest reward regardless of cost constraints.  Let  $\pi^*_c$ denote the optimal policy for a CMDP, which by definition does not violate constraints. By assumption, in at least some states, $R$ gives the highest reward actions that exceed the constraints $c_i(s, a) \ge b_i$. Therefore $\pi^*$ must take these high-reward, high-cost actions, while  $\pi^*_c$ must avoid them, ensuring that $\pi^*$ and $\pi^*_c$ do not coincide.

Now recall from \citet{skalse2022defining} that a pair of reward functions $R_1$, $R_2$ is hackable relative to a policy set $\Pi$ if there exists $\pi, \pi^{\prime} \in \Pi$ such that:\[
J(R_1, \pi) < J(R_1, \pi^{\prime}) \text{ and } J(R_2, \pi) > J(R_2, \pi^{\prime}),
\]where the $J(R, \pi)$ is the value of policy $\pi$ under reward function $R$ i.e. $J(R, \pi) = \sum_{(s, a)} \pi(a | s) R(s, a)$. Define $R_c$ as follows:\[
R_c(s, a) = \begin{cases}
R(s, a) & \text{if } c(s, a) \leq b ,\\
0 & \text{ otherwise.}
\end{cases}\]
Notice that in practice while $\pi^*_c$ would optimize for $R$, the actual reward it gets is the same as under $R_c$ since actions that violate constraints are never taken meaning that we can treat the reward that $\pi^*_c$ maximizes as $R_c$ i.e. the introduction of $R_c$ does not change the CMDP's optimization goal; it merely provides a reward-function-based representation of the same constraints making our assumption that $R_c$ is optimized valid. Under $R$, $\pi^*$ maximizes reward without regard to constraints so $ J(R, \pi^*) > J(R, \pi^*_c)$. Under $R_c$, any action violating the constraint is assigned zero reward ensuring $J(R_c, \pi^*) < J(R_c, \pi^*_c)$.  These inequalities satisfy the hacking definition. Thus solving the CMDP has the same outcome as optimizing a hackable reward function whenever the reward function rewarded high cost actions. This is the definition of reward hacking therefore Proposition~\ref{prop:1} holds.
\end{proof}

While we only have provided empirical evidence that learning constraints are active in the face of adversarial reward, though not perfectly, and therefore these defences are not optimal solutions to CMDPs we believe that Proposition~\ref{prop:1} could explain the qualitative analysis above as well as the reward analysis in Figure~\ref{fig:reward-analysis}. In this case, the hackable reward pair are the proxy harmful reward function and the true harmful distribution. Figure~\ref{fig:reward-analysis} shows that the proxy harmful reward function is maximized for RepNoise $\rightarrow$ TAR but clearly the true harmful distribution is not being optimized due to learning constraints. Regardless future defences that can be shown as optimal solutions to CMDPs would be subject to Proposition~\ref{prop:1} in the face of adversarial reward. Future defences could leverage this fact by being developed to intentionally exploit harmless discrepancies between the proxy harmful reward and the true harmful reward that instantiates the goals of the attacker.

\section{Conclusion}

 In this work, we find that LLMs are vulnerable to RPAs, which is concerning given the popularity of feedback collection mechanisms, preferences datasets, and the use of RLHF for training LLMs. Based on this finding, we develop a conceptual framework for adapting current defences against harmful fine-tuning attacks based on Constrained MDPs (CMDP). We unite various defences under a common framework of learning constraints both ``online'' or ``explicit'' and ``offline'' or ``implicit'', and show how they fit in the CMDP framework. We explore the efficacy of these defences under a variety of settings, finding that several online defences are generally protective against RPAs without detracting from the ability to learn a harmless preference task. While offline defences provide some protection, they are generally not effective against RPAs pointing to the necessity of future research when the defence assumption does not allow the defender to apply an online defence. 

 We have shown a surprising result that some defences which do mitigate optimizing harmful reward end up performing ``harmless reward hacking'' and explained why we might expect this to be the case by showing that optimal CMDPs and reward hacking have equivalent outcomes.

 Finally, we acknowledge that our study is limited in two respects. First, it doesn't account for the large body of previous work from Safe RL outside of the LLM domain especially in the case of solving CMDPs. Second, our study is mostly an empirical analysis of previous defences and we did not provide a theoretical analysis of whether the ``explicit'' or ``implicit'' defences provide optimal CMDP solutions. 


\section{Acknowledgments}
We thank Elahe Rahimi for reading an early version of this paper and providing valuable comments. We acknowledge the support of the Killam foundation, Vector Institute, the Natural Sciences and Engineering Research Council of Canada (NSERC), RGPIN-2022-03943, Canada Foundation of Innovation (CFI), Digital Research Alliance of Canada, and Research Nova Scotia.
\bibliography{aaai25}

\begin{thebibliography}{39}
\providecommand{\natexlab}[1]{#1}

\bibitem[{Altman(2021)}]{altman2021constrained}
Altman, E. 2021.
\newblock \emph{Constrained Markov decision processes}.
\newblock Routledge.

\bibitem[{Andriushchenko et~al.(2024)Andriushchenko, Souly, Dziemian, Duenas, Lin, Wang, Hendrycks, Zou, Kolter, Fredrikson et~al.}]{andriushchenko2024agentharm}
Andriushchenko, M.; Souly, A.; Dziemian, M.; Duenas, D.; Lin, M.; Wang, J.; Hendrycks, D.; Zou, A.; Kolter, Z.; Fredrikson, M.; et~al. 2024.
\newblock Agentharm: A benchmark for measuring harmfulness of llm agents.
\newblock \emph{arXiv preprint arXiv:2410.09024}.

\bibitem[{Bai et~al.(2022)Bai, Jones, Ndousse, Askell, Chen, DasSarma, Drain, Fort, Ganguli, Henighan, Joseph, Kadavath, Kernion, Conerly, El-Showk, Elhage, Hatfield-Dodds, Hernandez, Hume, Johnston, Kravec, Lovitt, Nanda, Olsson, Amodei, Brown, Clark, McCandlish, Olah, Mann, and Kaplan}]{bai2022traininghelpfulharmlessassistant}
Bai, Y.; Jones, A.; Ndousse, K.; Askell, A.; Chen, A.; DasSarma, N.; Drain, D.; Fort, S.; Ganguli, D.; Henighan, T.; Joseph, N.; Kadavath, S.; Kernion, J.; Conerly, T.; El-Showk, S.; Elhage, N.; Hatfield-Dodds, Z.; Hernandez, D.; Hume, T.; Johnston, S.; Kravec, S.; Lovitt, L.; Nanda, N.; Olsson, C.; Amodei, D.; Brown, T.; Clark, J.; McCandlish, S.; Olah, C.; Mann, B.; and Kaplan, J. 2022.
\newblock Training a Helpful and Harmless Assistant with Reinforcement Learning from Human Feedback.
\newblock arXiv:2204.05862.

\bibitem[{Boyd and Vandenberghe(2004)}]{boyd2004convex}
Boyd, S.; and Vandenberghe, L. 2004.
\newblock \emph{Convex optimization}.
\newblock Cambridge university press.

\bibitem[{Brunke et~al.(2022)Brunke, Greeff, Hall, Yuan, Zhou, Panerati, and Schoellig}]{brunke2022safe}
Brunke, L.; Greeff, M.; Hall, A.~W.; Yuan, Z.; Zhou, S.; Panerati, J.; and Schoellig, A.~P. 2022.
\newblock Safe learning in robotics: From learning-based control to safe reinforcement learning.
\newblock \emph{Annual Review of Control, Robotics, and Autonomous Systems}, 5(1): 411--444.

\bibitem[{Casper et~al.(2023)Casper, Davies, Shi, Gilbert, Scheurer, Rando, Freedman, Korbak, Lindner, Freire, Wang, Marks, Segerie, Carroll, Peng, Christoffersen, Damani, Slocum, Anwar, Siththaranjan, Nadeau, Michaud, Pfau, Krasheninnikov, Chen, Langosco, Hase, Bıyık, Dragan, Krueger, Sadigh, and Hadfield-Menell}]{casper2023openproblemsfundamentallimitations}
Casper, S.; Davies, X.; Shi, C.; Gilbert, T.~K.; Scheurer, J.; Rando, J.; Freedman, R.; Korbak, T.; Lindner, D.; Freire, P.; Wang, T.; Marks, S.; Segerie, C.-R.; Carroll, M.; Peng, A.; Christoffersen, P.; Damani, M.; Slocum, S.; Anwar, U.; Siththaranjan, A.; Nadeau, M.; Michaud, E.~J.; Pfau, J.; Krasheninnikov, D.; Chen, X.; Langosco, L.; Hase, P.; Bıyık, E.; Dragan, A.; Krueger, D.; Sadigh, D.; and Hadfield-Menell, D. 2023.
\newblock Open Problems and Fundamental Limitations of Reinforcement Learning from Human Feedback.
\newblock arXiv:2307.15217.

\bibitem[{Dai et~al.(2024)Dai, Pan, Sun, Ji, Xu, Liu, Wang, and Yang}]{dai2024safe}
Dai, J.; Pan, X.; Sun, R.; Ji, J.; Xu, X.; Liu, M.; Wang, Y.; and Yang, Y. 2024.
\newblock Safe {RLHF}: Safe Reinforcement Learning from Human Feedback.
\newblock In \emph{The Twelfth International Conference on Learning Representations}.

\bibitem[{Garc{\i}a and Fern{\'a}ndez(2015)}]{garcia2015comprehensive}
Garc{\i}a, J.; and Fern{\'a}ndez, F. 2015.
\newblock A comprehensive survey on safe reinforcement learning.
\newblock \emph{Journal of Machine Learning Research}, 16(1): 1437--1480.

\bibitem[{Gu et~al.(2024{\natexlab{a}})Gu, Yang, Du, Chen, Walter, Wang, and Knoll}]{gu2024review}
Gu, S.; Yang, L.; Du, Y.; Chen, G.; Walter, F.; Wang, J.; and Knoll, A. 2024{\natexlab{a}}.
\newblock A Review of Safe Reinforcement Learning: Methods, Theories and Applications.
\newblock \emph{IEEE Transactions on Pattern Analysis and Machine Intelligence}.

\bibitem[{Gu et~al.(2024{\natexlab{b}})Gu, Yang, Du, Chen, Walter, Wang, and Knoll}]{gu2024reviewsafereinforcementlearning}
Gu, S.; Yang, L.; Du, Y.; Chen, G.; Walter, F.; Wang, J.; and Knoll, A. 2024{\natexlab{b}}.
\newblock A Review of Safe Reinforcement Learning: Methods, Theory and Applications.
\newblock arXiv:2205.10330.

\bibitem[{Henderson et~al.(2023)Henderson, Mitchell, Manning, Jurafsky, and Finn}]{henderson2023self}
Henderson, P.; Mitchell, E.; Manning, C.; Jurafsky, D.; and Finn, C. 2023.
\newblock Self-destructing models: Increasing the costs of harmful dual uses of foundation models.
\newblock In \emph{Proceedings of the 2023 AAAI/ACM Conference on AI, Ethics, and Society}, 287--296.

\bibitem[{Hu et~al.()Hu, Wallis, Allen-Zhu, Li, Wang, Wang, Chen et~al.}]{hulora}
Hu, E.~J.; Wallis, P.; Allen-Zhu, Z.; Li, Y.; Wang, S.; Wang, L.; Chen, W.; et~al. ????
\newblock LoRA: Low-Rank Adaptation of Large Language Models.
\newblock In \emph{International Conference on Learning Representations}.

\bibitem[{Huang et~al.(2022)Huang, Dossa, Raffin, Kanervisto, and Wang}]{shengyi2022the37implementation}
Huang, S.; Dossa, R. F.~J.; Raffin, A.; Kanervisto, A.; and Wang, W. 2022.
\newblock The 37 Implementation Details of Proximal Policy Optimization.
\newblock In \emph{ICLR Blog Track}.
\newblock Https://iclr-blog-track.github.io/2022/03/25/ppo-implementation-details/.

\bibitem[{Huang et~al.(2024)Huang, Hu, Ilhan, Tekin, and Liu}]{Huang2024LazySA}
Huang, T.; Hu, S.; Ilhan, F.; Tekin, S.~F.; and Liu, L. 2024.
\newblock Lazy Safety Alignment for Large Language Models against Harmful Fine-tuning.
\newblock \emph{ArXiv}, abs/2405.18641.

\bibitem[{Huang, Hu, and Liu(2024)}]{Huang2024VaccinePA}
Huang, T.; Hu, S.; and Liu, L. 2024.
\newblock Vaccine: Perturbation-aware Alignment for Large Language Model.
\newblock \emph{ArXiv}, abs/2402.01109.

\bibitem[{Ji et~al.(2023)Ji, Liu, Dai, Pan, Zhang, Bian, Chen, Sun, Wang, and Yang}]{ji2023beavertails}
Ji, J.; Liu, M.; Dai, J.; Pan, X.; Zhang, C.; Bian, C.; Chen, B.; Sun, R.; Wang, Y.; and Yang, Y. 2023.
\newblock BeaverTails: Towards Improved Safety Alignment of {LLM} via a Human-Preference Dataset.
\newblock In \emph{Thirty-seventh Conference on Neural Information Processing Systems Datasets and Benchmarks Track}.

\bibitem[{Kang et~al.(2022)Kang, Gradu, Choi, Janner, Tomlin, and Levine}]{kang2022lyapunov}
Kang, K.; Gradu, P.; Choi, J.~J.; Janner, M.; Tomlin, C.; and Levine, S. 2022.
\newblock Lyapunov density models: Constraining distribution shift in learning-based control.
\newblock In \emph{International Conference on Machine Learning}, 10708--10733. PMLR.

\bibitem[{Lermen, Rogers-Smith, and Ladish(2024)}]{lermen2024lorafinetuningefficientlyundoes}
Lermen, S.; Rogers-Smith, C.; and Ladish, J. 2024.
\newblock LoRA Fine-tuning Efficiently Undoes Safety Training in Llama 2-Chat 70B.
\newblock arXiv:2310.20624.

\bibitem[{Li et~al.(2024{\natexlab{a}})Li, Pan, Gopal, Yue, Berrios, Gatti, Li, Dombrowski, Goel, Phan, Mukobi, Helm-Burger, Lababidi, Justen, Liu, Chen, Barrass, Zhang, Zhu, Tamirisa, Bharathi, Khoja, Zhao, Herbert-Voss, Breuer, Marks, Patel, Zou, Mazeika, Wang, Oswal, Lin, Hunt, Tienken-Harder, Shih, Talley, Guan, Kaplan, Steneker, Campbell, Jokubaitis, Levinson, Wang, Qian, Karmakar, Basart, Fitz, Levine, Kumaraguru, Tupakula, Varadharajan, Wang, Shoshitaishvili, Ba, Esvelt, Wang, and Hendrycks}]{li2024wmdpbenchmarkmeasuringreducing}
Li, N.; Pan, A.; Gopal, A.; Yue, S.; Berrios, D.; Gatti, A.; Li, J.~D.; Dombrowski, A.-K.; Goel, S.; Phan, L.; Mukobi, G.; Helm-Burger, N.; Lababidi, R.; Justen, L.; Liu, A.~B.; Chen, M.; Barrass, I.; Zhang, O.; Zhu, X.; Tamirisa, R.; Bharathi, B.; Khoja, A.; Zhao, Z.; Herbert-Voss, A.; Breuer, C.~B.; Marks, S.; Patel, O.; Zou, A.; Mazeika, M.; Wang, Z.; Oswal, P.; Lin, W.; Hunt, A.~A.; Tienken-Harder, J.; Shih, K.~Y.; Talley, K.; Guan, J.; Kaplan, R.; Steneker, I.; Campbell, D.; Jokubaitis, B.; Levinson, A.; Wang, J.; Qian, W.; Karmakar, K.~K.; Basart, S.; Fitz, S.; Levine, M.; Kumaraguru, P.; Tupakula, U.; Varadharajan, V.; Wang, R.; Shoshitaishvili, Y.; Ba, J.; Esvelt, K.~M.; Wang, A.; and Hendrycks, D. 2024{\natexlab{a}}.
\newblock The WMDP Benchmark: Measuring and Reducing Malicious Use With Unlearning.
\newblock arXiv:2403.03218.

\bibitem[{Li et~al.(2024{\natexlab{b}})Li, Pan, Gopal, Yue, Berrios, Gatti, Li, Dombrowski, Goel, Phan, Mukobi, Helm-Burger, Lababidi, Justen, Liu, Chen, Barrass, Zhang, Zhu, Tamirisa, Bharathi, Khoja, Zhao, Herbert-Voss, Breuer, Marks, Patel, Zou, Mazeika, Wang, Oswal, Liu, Hunt, Tienken-Harder, Shih, Talley, Guan, Kaplan, Steneker, Campbell, Jokubaitis, Levinson, Wang, Qian, Karmakar, Basart, Fitz, Levine, Kumaraguru, Tupakula, Varadharajan, Shoshitaishvili, Ba, Esvelt, Wang, and Hendrycks}]{li2024wmdp}
Li, N.; Pan, A.; Gopal, A.; Yue, S.; Berrios, D.; Gatti, A.; Li, J.~D.; Dombrowski, A.-K.; Goel, S.; Phan, L.; Mukobi, G.; Helm-Burger, N.; Lababidi, R.; Justen, L.; Liu, A.~B.; Chen, M.; Barrass, I.; Zhang, O.; Zhu, X.; Tamirisa, R.; Bharathi, B.; Khoja, A.; Zhao, Z.; Herbert-Voss, A.; Breuer, C.~B.; Marks, S.; Patel, O.; Zou, A.; Mazeika, M.; Wang, Z.; Oswal, P.; Liu, W.; Hunt, A.~A.; Tienken-Harder, J.; Shih, K.~Y.; Talley, K.; Guan, J.; Kaplan, R.; Steneker, I.; Campbell, D.; Jokubaitis, B.; Levinson, A.; Wang, J.; Qian, W.; Karmakar, K.~K.; Basart, S.; Fitz, S.; Levine, M.; Kumaraguru, P.; Tupakula, U.; Varadharajan, V.; Shoshitaishvili, Y.; Ba, J.; Esvelt, K.~M.; Wang, A.; and Hendrycks, D. 2024{\natexlab{b}}.
\newblock The WMDP Benchmark: Measuring and Reducing Malicious Use With Unlearning.
\newblock arXiv:2403.03218.

\bibitem[{Pan et~al.(2023)Pan, Chan, Zou, Li, Basart, Woodside, Ng, Zhang, Emmons, and Hendrycks}]{pan2023machiavelli}
Pan, A.; Chan, J.~S.; Zou, A.; Li, N.; Basart, S.; Woodside, T.; Ng, J.; Zhang, H.; Emmons, S.; and Hendrycks, D. 2023.
\newblock Do the Rewards Justify the Means? Measuring Trade-Offs Between Rewards and Ethical Behavior in the Machiavelli Benchmark.
\newblock \emph{ICML}.

\bibitem[{Qi et~al.(2024{\natexlab{a}})Qi, Panda, Lyu, Ma, Roy, Beirami, Mittal, and Henderson}]{qi2024safetyalignmentjusttokens}
Qi, X.; Panda, A.; Lyu, K.; Ma, X.; Roy, S.; Beirami, A.; Mittal, P.; and Henderson, P. 2024{\natexlab{a}}.
\newblock Safety Alignment Should Be Made More Than Just a Few Tokens Deep.
\newblock arXiv:2406.05946.

\bibitem[{Qi et~al.(2024{\natexlab{b}})Qi, Zeng, Xie, Chen, Jia, Mittal, and Henderson}]{qifine}
Qi, X.; Zeng, Y.; Xie, T.; Chen, P.-Y.; Jia, R.; Mittal, P.; and Henderson, P. 2024{\natexlab{b}}.
\newblock Fine-tuning Aligned Language Models Compromises Safety, Even When Users Do Not Intend To!
\newblock In \emph{The Twelfth International Conference on Learning Representations}.

\bibitem[{Radford et~al.(2019)Radford, Wu, Child, Luan, Amodei, Sutskever et~al.}]{radford2019language}
Radford, A.; Wu, J.; Child, R.; Luan, D.; Amodei, D.; Sutskever, I.; et~al. 2019.
\newblock Language models are unsupervised multitask learners.
\newblock \emph{OpenAI blog}, 1(8): 9.

\bibitem[{Rafailov et~al.(2023)Rafailov, Sharma, Mitchell, Manning, Ermon, and Finn}]{rafailov2023direct}
Rafailov, R.; Sharma, A.; Mitchell, E.; Manning, C.~D.; Ermon, S.; and Finn, C. 2023.
\newblock Direct Preference Optimization: Your Language Model is Secretly a Reward Model.
\newblock In \emph{Thirty-seventh Conference on Neural Information Processing Systems}.

\bibitem[{Rando and Tram{\`e}r(2023)}]{rando2023universal}
Rando, J.; and Tram{\`e}r, F. 2023.
\newblock Universal jailbreak backdoors from poisoned human feedback.
\newblock \emph{arXiv preprint arXiv:2311.14455}.

\bibitem[{Rosati et~al.(2024{\natexlab{a}})Rosati, Wehner, Williams, Łukasz Bartoszcze, Atanasov, Gonzales, Majumdar, Maple, Sajjad, and Rudzicz}]{rosati2024representationnoisingeffectivelyprevents}
Rosati, D.; Wehner, J.; Williams, K.; Łukasz Bartoszcze; Atanasov, D.; Gonzales, R.; Majumdar, S.; Maple, C.; Sajjad, H.; and Rudzicz, F. 2024{\natexlab{a}}.
\newblock Representation noising effectively prevents harmful fine-tuning on LLMs.
\newblock arXiv:2405.14577.

\bibitem[{Rosati et~al.(2024{\natexlab{b}})Rosati, Wehner, Williams, Łukasz Bartoszcze, Batzner, Sajjad, and Rudzicz}]{rosati2024immunizationharmfulfinetuningattacks}
Rosati, D.; Wehner, J.; Williams, K.; Łukasz Bartoszcze; Batzner, J.; Sajjad, H.; and Rudzicz, F. 2024{\natexlab{b}}.
\newblock Immunization against harmful fine-tuning attacks.
\newblock arXiv:2402.16382.

\bibitem[{Schulman et~al.(2017)Schulman, Wolski, Dhariwal, Radford, and Klimov}]{schulman2017proximalpolicyoptimizationalgorithms}
Schulman, J.; Wolski, F.; Dhariwal, P.; Radford, A.; and Klimov, O. 2017.
\newblock Proximal Policy Optimization Algorithms.
\newblock arXiv:1707.06347.

\bibitem[{Skalse et~al.(2022)Skalse, Howe, Krasheninnikov, and Krueger}]{skalse2022defining}
Skalse, J.; Howe, N.; Krasheninnikov, D.; and Krueger, D. 2022.
\newblock Defining and characterizing reward gaming.
\newblock \emph{Advances in Neural Information Processing Systems}, 35: 9460--9471.

\bibitem[{Stiennon et~al.(2020)Stiennon, Ouyang, Wu, Ziegler, Lowe, Voss, Radford, Amodei, and Christiano}]{stiennon2020learning}
Stiennon, N.; Ouyang, L.; Wu, J.; Ziegler, D.; Lowe, R.; Voss, C.; Radford, A.; Amodei, D.; and Christiano, P.~F. 2020.
\newblock Learning to summarize with human feedback.
\newblock \emph{Advances in Neural Information Processing Systems}, 33: 3008--3021.

\bibitem[{Tamirisa et~al.(2024)Tamirisa, Bharathi, Phan, Zhou, Gatti, Suresh, Lin, Wang, Wang, Arel et~al.}]{tamirisa2024tamper}
Tamirisa, R.; Bharathi, B.; Phan, L.; Zhou, A.; Gatti, A.; Suresh, T.; Lin, M.; Wang, J.; Wang, R.; Arel, R.; et~al. 2024.
\newblock Tamper-Resistant Safeguards for Open-Weight LLMs.
\newblock \emph{arXiv preprint arXiv:2408.00761}.

\bibitem[{Touvron et~al.(2023)Touvron, Martin, Stone, Albert, Almahairi, Babaei, Bashlykov, Batra, Bhargava, Bhosale, Bikel, Blecher, Ferrer, Chen, Cucurull, Esiobu, Fernandes, Fu, Fu, Fuller, Gao, Goswami, Goyal, Hartshorn, Hosseini, Hou, Inan, Kardas, Kerkez, Khabsa, Kloumann, Korenev, Koura, Lachaux, Lavril, Lee, Liskovich, Lu, Mao, Martinet, Mihaylov, Mishra, Molybog, Nie, Poulton, Reizenstein, Rungta, Saladi, Schelten, Silva, Smith, Subramanian, Tan, Tang, Taylor, Williams, Kuan, Xu, Yan, Zarov, Zhang, Fan, Kambadur, Narang, Rodriguez, Stojnic, Edunov, and Scialom}]{touvron2023llama2openfoundation}
Touvron, H.; Martin, L.; Stone, K.; Albert, P.; Almahairi, A.; Babaei, Y.; Bashlykov, N.; Batra, S.; Bhargava, P.; Bhosale, S.; Bikel, D.; Blecher, L.; Ferrer, C.~C.; Chen, M.; Cucurull, G.; Esiobu, D.; Fernandes, J.; Fu, J.; Fu, W.; Fuller, B.; Gao, C.; Goswami, V.; Goyal, N.; Hartshorn, A.; Hosseini, S.; Hou, R.; Inan, H.; Kardas, M.; Kerkez, V.; Khabsa, M.; Kloumann, I.; Korenev, A.; Koura, P.~S.; Lachaux, M.-A.; Lavril, T.; Lee, J.; Liskovich, D.; Lu, Y.; Mao, Y.; Martinet, X.; Mihaylov, T.; Mishra, P.; Molybog, I.; Nie, Y.; Poulton, A.; Reizenstein, J.; Rungta, R.; Saladi, K.; Schelten, A.; Silva, R.; Smith, E.~M.; Subramanian, R.; Tan, X.~E.; Tang, B.; Taylor, R.; Williams, A.; Kuan, J.~X.; Xu, P.; Yan, Z.; Zarov, I.; Zhang, Y.; Fan, A.; Kambadur, M.; Narang, S.; Rodriguez, A.; Stojnic, R.; Edunov, S.; and Scialom, T. 2023.
\newblock Llama 2: Open Foundation and Fine-Tuned Chat Models.
\newblock arXiv:2307.09288.

\bibitem[{Yang et~al.(2024)Yang, Pan, Luo, Qiu, Zhong, Yu, and Chen}]{yang2024rewards}
Yang, R.; Pan, X.; Luo, F.; Qiu, S.; Zhong, H.; Yu, D.; and Chen, J. 2024.
\newblock Rewards-in-Context: Multi-objective Alignment of Foundation Models with Dynamic Preference Adjustment.
\newblock \emph{International Conference on Machine Learning}.

\bibitem[{Yang et~al.(2023)Yang, Wang, Zhang, Petzold, Wang, Zhao, and Lin}]{yang2023shadowalignmenteasesubverting}
Yang, X.; Wang, X.; Zhang, Q.; Petzold, L.; Wang, W.~Y.; Zhao, X.; and Lin, D. 2023.
\newblock Shadow Alignment: The Ease of Subverting Safely-Aligned Language Models.
\newblock arXiv:2310.02949.

\bibitem[{Yi et~al.(2024)Yi, Ye, Chen, Zhu, Chen, Lian, Sun, Xie, and Wu}]{yi2024opensource}
Yi, J.; Ye, R.; Chen, Q.; Zhu, B.~B.; Chen, S.; Lian, D.; Sun, G.; Xie, X.; and Wu, F. 2024.
\newblock Open-Source Can Be Dangerous: On the Vulnerability of Value Alignment in Open-Source {LLM}s.

\bibitem[{Zhan et~al.(2024)Zhan, Fang, Bindu, Gupta, Hashimoto, and Kang}]{zhan-etal-2024-removing}
Zhan, Q.; Fang, R.; Bindu, R.; Gupta, A.; Hashimoto, T.; and Kang, D. 2024.
\newblock Removing {RLHF} Protections in {GPT}-4 via Fine-Tuning.
\newblock In Duh, K.; Gomez, H.; and Bethard, S., eds., \emph{Proceedings of the 2024 Conference of the North American Chapter of the Association for Computational Linguistics: Human Language Technologies (Volume 2: Short Papers)}, 681--687. Mexico City, Mexico: Association for Computational Linguistics.

\bibitem[{Zhou et~al.(2023)Zhou, Lu, Ma, Gui, Zhang, and Huang}]{zhou2023makingharmfulbehaviorsunlearnable}
Zhou, X.; Lu, Y.; Ma, R.; Gui, T.; Zhang, Q.; and Huang, X. 2023.
\newblock Making Harmful Behaviors Unlearnable for Large Language Models.
\newblock arXiv:2311.02105.

\bibitem[{Zou et~al.(2024)Zou, Phan, Wang, Duenas, Lin, Andriushchenko, Wang, Kolter, Fredrikson, and Hendrycks}]{zou2024improvingalignmentrobustnesscircuit}
Zou, A.; Phan, L.; Wang, J.; Duenas, D.; Lin, M.; Andriushchenko, M.; Wang, R.; Kolter, Z.; Fredrikson, M.; and Hendrycks, D. 2024.
\newblock Improving Alignment and Robustness with Circuit Breakers.
\newblock arXiv:2406.04313.

\end{thebibliography}


\appendix

\section{Implementation Details}
\label{app:implementation-details}

The code to replicate this paper including reference implementation for all the defences listed in this paper are available here: \url{https://anonymous.4open.science/r/representation-noising-xpo-E351}

Experiments were run on 4xA40 (40 GB vRAM) GPUs, 4xA100 (80 GB vRAM) GPUs, and 4xH100 (80 GB vRAM) GPUs hardware.

\paragraph{Reward Model Training} For the harmfulness reward model we use \texttt{gpt2-medium} \citep{radford2019language} and train for 4 epochs on both 9,080 samples from SafeRLHF and 18,106 from BeaverTails datasets using a batch size of 8, Adam optimizer, and learning rate of $6 \times 10^{-6}$. For the varying flipping ratio experiments we train a new reward model for each ratio.

For the helpfulness reward model we leveraged the pre-trained model from \citet{yang2024rewards} which was trained on the helpfulness split of \cite{bai2022traininghelpfulharmlessassistant}.

\paragraph{RPA Implementation Details} For the reference models used, unless otherwise mentioned all models use a copy of their original model. For all experiments we use \texttt{llama2-7b-chat} \citep{touvron2023llama2openfoundation}. We use the TRL library \footnote{\url{https://huggingface.co/docs/trl}} to implement all methods and use the default hyperparameters for all methods. For PPO this is a learning rate of $1.41e-5$ and the use of an adaptive KL controller constraint with initial KL constraint of $0.2$. For DPO the default learning rate is $5e-5$  and the $\beta=0.1$. For SFT we use a learning rate of $3e-5$ and train on either the harmful or mixed answer completions. 

\paragraph{Harmfulness Classifier and Evaluation} For the harmfulness classifier we follow the same approach as \citep{rosati2024representationnoisingeffectivelyprevents}. We use \texttt{deberta-v3-xsmall} and train for 4 epochs using a batch size of 128, Adam optimizer, and learning rate of $6 \times 10^{-6}$. Our classifier achieves an F1 score of 0.87. We use 18,106 samples for training from the \citet{ji2023beavertails} paired with refusals generated from \texttt{llama2-7b-chat}. The task is to determine, given a harmful question (from \citealp{ji2023beavertails, dai2024safe}) such as "How to build a gun", the harmfulness of an answer. The \textbf{harmfulness scores} presented in the paper are the mean probability of assigning the harmful label to the question answer pair. For all evaluations we use the same held out 100 samples. Note that we do not consider other harmful tasks like toxicity classification and we do not use the harmlessness RLHF dataset from \citet{bai2022traininghelpfulharmlessassistant} since it is a multi-turn dialogue setting which is difficult to assess without the use of a LLM-as-judge with frontier models which we wanted to avoid for cost and experimental validity reasons. Additionally there are concerns \citep{ji2023beavertails} about quality of this dataset and questions about whether it accurately represents harmlessness so we chose to avoid it.

\begin{table*}[]
\centering
\begin{tabular}{lllccc}
\toprule
\textbf{Attack} & \textbf{Dataset} & \textbf{Defence} & \textbf{2.5k} & \textbf{5k} & \textbf{10k} \\
\midrule
\multirow[t]{6}{*}{DPO} & \multirow[t]{3}{*}{BeaverTails} & Lisa & 0.18 & 0.19 & 0.11 \\
 &  & RepNoise & 0.56 & 0.55 & 0.33 \\
 &  & RepNoise $\rightarrow$ TAR & 0.21 & 0.40 & 0.01 \\
 &  & Refusal & 0.06 & 0.07 & 0.07 \\
\cline{2-6} \noalign{\vspace{0.25ex}}
 & \multirow[t]{3}{*}{Safe RLHF} & Lisa & 0.05 & 0.05 & 0.06 \\
 &  & RepNoise & 0.86 & 0.86 & 0.86 \\
 &  & RepNoise $\rightarrow$ TAR & 0.93 & 0.67 & 0.88 \\
 &  & Refusal & 0.06 & 0.05 & 0.06 \\
\cline{1-6} \cline{2-6} \noalign{\vspace{0.25ex}}
\multirow[t]{5}{*}{PPO} & \multirow[t]{2}{*}{BeaverTails} & Lisa & 0.08 & 0.07 & 0.13 \\
 &  & RepNoise & -- & -- & -- \\
 & & RepNoise $\rightarrow$ TAR & 0.44 & 0.00 & 0.22 \\
 &  & Refusal & 0.07 & 0.14 & 0.60 \\
\cline{2-6} \noalign{\vspace{0.25ex}}
 & \multirow[t]{3}{*}{Safe RLHF} & Lisa & 0.08 & 0.06 & 0.06 \\
 &  & RepNoise & 0.48 & -- & -- \\
 &  & RepNoise $\rightarrow$ TAR & -- & -- & 0.51 \\
 &  & Refusal & 0.06 & 0.06 & 0.09 \\
\cline{1-6} \cline{2-6} \noalign{\vspace{0.25ex}}
\multirow[t]{6}{*}{SFT} & \multirow[t]{3}{*}{BeaverTails} & Lisa & 0.31 & 0.60 & 0.61 \\
 &  & RepNoise & 0.33 & 0.43 & 0.43 \\
 &  & RepNoise $\rightarrow$ TAR & 0.18 & 0.29 & 0.31 \\
 &  & Refusal & 0.07 & 0.10 & 0.12 \\
\cline{2-6} \noalign{\vspace{0.25ex}}
 & \multirow[t]{3}{*}{Safe RLHF} & Lisa & 0.16 & 0.44 & 0.28 \\
 &  & RepNoise & 0.39 & 0.50 & 0.50 \\
  &  & RepNoise $\rightarrow$ TAR & 0.52 & 0.58 & 0.60 \\
 &  & Refusal & 0.07 & 0.10 & 0.09 \\
\bottomrule
\end{tabular}
\caption{
    \label{tab:stronger-attack} An analysis of stronger attack settings on varying sizes of harmful samples.
}
\end{table*}

\subsection{Defences}
\label{app:defence}

\paragraph{Lisa and Vaccine} For Lisa \citep{Huang2024LazySA} and Vaccine \citep{Huang2024VaccinePA}, we develop our own implementations of these and based on hyperparameter tuning experiments (see Table~\ref{tab:hyperparam-variation}) we set the $\rho$ of both of these to 100. Both methods constrain the allowable discrepancy between the policy model we are learning and the original policy either through an embedding constraint (Vaccine) or a weight value constraint (Lisa). Any $\rho$ below 100 results resulted in much pooer defences illustrated in the Appendix.

\paragraph{Security Vectors} The method presented in \citet{zhou2023makingharmfulbehaviorsunlearnable} is based on the conjecture that if the loss is already very low for harmfulness then very little parameter changes will take place during training. Therefore they construct a vector using LoRA \citep{hulora} which is trained to be harmful and then apply this vector to the model whenever downstream training occurs but remove the vector during inference. To make the setting as fair as possible we use the same number of gradient steps as our offline methods with the same defence training set up which will be described below in the Appendix.

\paragraph{SafeRLHF} This method from \citet{dai2024safe} is our only traditional Safe RL method which learns a reward shaping function based on an explicit cost constraint to satisfy the CMDP setting. Unfortunately this method is only applicable to PPO. We found this method is generally not very effective across initial $\lambda$ parameters in our RPA setting (evaluated in the range of $\lambda \in \{0.1, 1\}$ where the $\lambda$ is the learned Lagrange multiplier on the cost constraint. we used our harmfulness classifier above as the cost constraint. We also experimented with regular reward shaping as they did in the paper but this method was even less effective than SafeRLHF the method.

\paragraph{Refusal Loss} This is the simple baseline proposed by \cite{qi2024safetyalignmentjusttokens} where we simply add an auxiliary loss term that performs casual language modeling on safety samples. The auxiliary loss is weighted by the $\alpha$ parameter and unlike \cite{qi2024safetyalignmentjusttokens} we found this $\alpha$ had to be very high (see Table~\ref{tab:hyperparam-variation}) which is the main downside of the method despite its effectiveness and simplicity. For the Refusal dataset we selected the BeaverTail refusals generated from \texttt{llama2-7b-chat} which we described above. Note that unlike in  \cite{qi2024safetyalignmentjusttokens} which presents the multiplier as $\alpha$ and $(1 - \alpha)$ for the refusal and original loss, we present the multiplier as simply and integer $alpha \in \mathcal{Z}$ with the original loss with a $1$ multiplier in order to simplify the analysis comparison with Lisa in the Appendic.

\paragraph{Representation Noising Training Details} The representation noising defence is trained using 7,016 paired samples from BeaverTails harmful QA task \citep{ji2023beavertails}. We use the same procedure as \citep{rosati2024representationnoisingeffectivelyprevents} with $\alpha=1$ and $\beta=0.001$ using a batch size of 8 for 1 epoch using a learning rate of $4e-5$.

Previous works on implicit constraints \citep{li2024wmdpbenchmarkmeasuringreducing,zou2024improvingalignmentrobustnesscircuit,rosati2024representationnoisingeffectivelyprevents} can be considered to be specific types of a general representation removal algorithm which has the following structure: a projection function $g(\cdot)$ and a distance function $d(\cdot, \cdot)$. The goal of the resulting loss function is to minimize the distance between a set of representations of harmful text $X_{harm}$ as represented by the activations $Z_{harm}$ of the neural network at a given layer $l_i$ and some projection of those representations $g(Z_{harm})$: $min_{Z_{harm}} \: d(Z_{harm}, g(Z_{harm}))$. The goal of the projection is such that the representations $Z_{harm}$ minimize the mutual information with harmful text outputs $\min \:I(Z_{harm};Y_{harm})$ since when $Z$ and $Y_{harm}$ are independent, by definition, $Z$ cannot assist in the prediction of harmful tokens in $Y_{harm}$ (see \citealp{rosati2024representationnoisingeffectivelyprevents} for more details). The implicit constraint then is the fact that generating harmful tokens necessary to explore the harmful action space to achieve high reward is as unlikely as possible which results in slowing down the RLHF learning algorithm.

To implement representation removal, \citet{rosati2024representationnoisingeffectivelyprevents} uses a noise projection function and a distributional distance, maximum mean discrepancy, based on samples drawn from a harmful question answering dataset. Circuit breakers \citep{zou2024improvingalignmentrobustnesscircuit} use the original aligned model activations to minimize a cosine similarity distance function resulting in orthogonal projections. RMU \citep{li2024wmdpbenchmarkmeasuringreducing} uses a fixed noise vector as the projection with Euclidean distance.  However, these explorations hardly exhaust the space of distance and projection functions; future work could use this framework to develop methods with better theoretical guarantees of $\min\: I(Z_{harm};Y_{harm})$ as an implicit constraint. We also use a meta-learned implicit constraint introduced by \citet{tamirisa2024tamper} which extends \citet{henderson2023self} to the generative setting. Here, the implicit constraint is learned through the meta-learning process and is not transparent to the defender. Meta-learned constraints and representation removal-based constraints could easily be paired, which we show below.

\paragraph{Other Representation Removal Methods} RMU \citep{li2024wmdp} and Circuit Breakers \citep{zou2024improvingalignmentrobustnesscircuit} differ slightly from the framework RepNoise on two important details: (1) The retain datasets (i.e. refusals to answer harmful questions) they use are used with a representation level loss instead of regular casual language modeling on the retain samples. (2) They do not implement an additional gradient ascent loss. Based on early experiments we found that using a retain loss with casual language modeling and adding a gradient ascent loss both improved these methods so the same set up as RepNoise for these is used for RMU and Circuit Breakers. Circuit Breakers has a few additional implementation details such as using a LoRA adapter during training. For fairness we ran the original Circuit Breaker setting with the same hyperparameter settings as the paper and found this method was inferior to our full fine-tuning method.

\paragraph{Meta-learning approaches} Meta-learning methods are promising since they directly learn the optimal policy weights such that the policy is difficult to train towards a harmful end. Unfortunately, these methods are both very computational expensive and do not provide insights on why a defence might work like representation removal approaches do since the defence must be directly learned. We used \cite{tamirisa2024tamper} extension of \cite{henderson2023self} to evaluate meta-learning. In order to make the evaluation fair we tried to perform approximately the same gradient steps as above. This means that we performed 100 outer loop epochs with 4 inner loop rollouts each performing an attack with 64 training steps at a batch size of $8$. This is many more gradient steps than the above defences and the run time was approximately 8 times the amount of wall clock time required for the representation removal defences but we only update the actual model parameters at each outer loop step. We performed parameter tuning across the following learning rates $\{1e-4, 5e-5, 2e-5\}$ and found the best results with $5e-5$. As we mentioned in the main text, meta-learning approaches can easily be combined with representation removal methods. We develop the two following approaches (1) TAR+RepNoise: where the full RepNoise loss is applied during the tamper resistant steps in the inner loop which means that we to learn training trajectories that minimize the mutual information of harmful representations and harmful text outputs (2) RepNoise $\rightarrow$ TAR which performs the TAR approach after RepNoise which we found the most effective. The explanation for the effectiveness of (2) might be that RepNoise finds a local safety minima and TAR makes this minima hard to escape from. This speculation needs to be investigated by future work but if this is the case then a simpler gradient penalty loss term $||\nabla \mathcal{L}||_2$ could achieve the same end in a much more computationally efficient and explainable manner.

\begin{table*}[t!]
\centering
\begin{tabular}{llcccc}
\toprule
\textbf{Method} & \textbf{Defence} & \textbf{Reward $\uparrow$} & 
\textbf{PPL $\downarrow$} & \textbf{Length $\downarrow$} & \textbf{ROUGE-1 $\uparrow$} \\
\midrule
Pre & None & -4.76 & 23.38 & 39.19 & 0.19 \\
\cline{1-6} \noalign{\vspace{0.25ex}}
\multirow[t]{3}{*}{DPO} & Lisa & -2.70 & 54.80 & 32.65 & 0.18 \\
 & RepNoise & -2.67 & 57.52 & 33.88 & 0.17 \\
 & Refusal Loss & -2.86 & 32.52 & 36.06 & 0.24 \\
 & None & -2.67 & 45.40 & 33.12 & 0.19 \\
\cline{1-6} \noalign{\vspace{0.25ex}}
\multirow[t]{4}{*}{PPO} & Lisa & -2.64 & 30.54 & 36.42 & 0.22 \\
 & RepNoise & -2.64 & 41.10 & 36.07 & 0.22 \\
 & Refusal Loss & -2.86 & 32.52 & 36.06 & 0.24 \\
  & None & -2.57 & 39.41 & 34.11 & 0.20 \\
\cline{1-6} \noalign{\vspace{0.25ex}}
\multirow[t]{3}{*}{SFT} & Lisa & -2.81 & 31.11 & 33.07 & 0.22 \\
 & RepNoise & -2.75 & 32.41 & 33.18 & 0.22 \\
 & Refusal Loss & -2.89 & 33.59 & 33.16 & 0.23 \\
 & None & -2.80 & 30.26 & 33.04 & 0.22 \\
\bottomrule
\end{tabular}
\caption{
    \label{tab:tldr}
    Analysis of defences on the harmless TL;DR summarization task.
}
\end{table*}

\section{Stronger Attacks}

The attacks presented above are small considering the size of preference datasets, which may reach into the millions of data points \citep{dai2024safe}. To simulate stronger attacks within out computational budget, we evaluate Lisa, Refusal Loss, RepNoise, and RepNoise $\rightarrow$ TAR across HFTA and RPA with attack datasets of high sample sizes.  We keep the other settings such as learning rate, optimizer identical for all dataset sizes. 
The results (Table~\ref{tab:stronger-attack}) show that, except for RepNoise with Safe RLHF DPO, all methods provide some measure of defence against RPAs as measured by generating responses that are less harmful than a successfully attacked model (see Table~\ref{tab:vulnerability-method}). Unsurprisingly, online methods work better than offline methods with Refusal Loss as the most successful method only breaking over the course of a large PPO attack. Given our observations in Figure~\ref{fig:reward-analysis}, this score could simply reflect that the training was stopped before the Refusal Loss was optimized again.

Recall that we see in Table~\ref{tab:vulnerability-mixture} that RPAs can be effective with flipping a small percentage of labels. To prevent against such attacks based on partial label flipping, both RepNoise and Refusal Loss are generally effective (Table~\ref{tab:defence-mixture} in the Appendix). Additionally, on the unaligned \texttt{llama2-7b} we find that Refusal Loss is able to provide an effective defence, even when 90\% of the labels being flipped against learning to be misaligned and  ends up learning a safety guard. For reference, the original harmfulness scores of the unaligned \texttt{llama2-7b} are 0.55 (Safe RLHF) and 0.54 (BeaverTails).

\section{Learning a Harmless Task}

For defences to be reliable they must also allow training on harmless tasks. Without this condition there would be social pressure to undo safety alignment that disallow training so that the research and commercial communities can continue to leverage the transfer learning capabilities of LLMs. We evaluate the most effective defences from above on a popular RLHF task that is unrelated to harmfulness: the TL;DR summarization task introduced by \citet{stiennon2020learning}. After the original model is defended using each defence, we train it further using use 1,000 random samples from their TL;DR summarization dataset with a batch size of 8. Rest of the attack settings are the same as elsewhere in this paper (Appendix: Implementation Details).

In Table~\ref{tab:tldr}, we observe that none of the proposed defences prevent reward model optimization itself. For all methods, this results in learning summarization by producing shorter passages (which are generally preferred in ground truth datasets). Not all defences are equally effective at maintaining a small perplexity and improving \texttt{ROUGE-1} scores. The Refusal Loss is generally the most effective method in improving \texttt{ROUGE-1}. This is true even for the case of RPA/DPO where all other methods (including the original method) suffer from the inability to increase \texttt{ROUGE-1} scores.

\begin{table*}[t!]
\centering
\begin{tabular}{lllccccc}
\toprule
\textbf{Model} & \textbf{Dataset} & \textbf{Attack} & \textbf{0.1} & \textbf{0.25} & \textbf{0.5} & \textbf{0.75} & \textbf{0.9} \\
\midrule
\multirow[t]{4}{*}{RepNoise} & \multirow[t]{2}{*}{BeaverTails} & DPO & 0.05 & 0.05 & 0.05 & 0.02 & 0.23 \\
 &  & SFT & 0.06 & 0.06 & 0.10 & 0.20 & 0.41 \\
\cline{2-8} \noalign{\vspace{0.25ex}}
 & \multirow[t]{2}{*}{Safe RLHF} & DPO & 0.39 & 0.02 & 0.39 & 0.45 & 0.01 \\
 &  & SFT & 0.23 & 0.23 & 0.24 & 0.27 & 0.30 \\
\cline{1-8} \cline{2-8}
\multirow[t]{6}{*}{Refusal Loss} & \multirow[t]{3}{*}{BeaverTails} & DPO & 0.05 & 0.05 & 0.06 & 0.06 & 0.06 \\
 &  & PPO & 0.14 & 0.09 & 0.06 & 0.06 & 0.07 \\
 &  & SFT & 0.06 & 0.07 & 0.06 & 0.06 & 0.06 \\
\cline{2-8} \noalign{\vspace{0.25ex}}
 & \multirow[t]{3}{*}{Safe RLHF} & DPO & 0.05 & 0.05 & 0.05 & 0.05 & 0.05 \\
 &  & PPO & 0.08 & 0.09 & 0.07 & 0.08 & 0.08 \\
 &  & SFT & 0.06 & 0.06 & 0.06 & 0.06 & 0.06 \\
\cline{1-8} \cline{2-8} \noalign{\vspace{0.25ex}}
\multirow[t]{6}{*}{Refusal Loss (unaligned)} & \multirow[t]{3}{*}{BeaverTails} & DPO & 0.05 & 0.05 & 0.05 & 0.05 & 0.05 \\
 &  & PPO & 0.24 & 0.20 & 0.18 & 0.10 & 0.12 \\
 &  & SFT & 0.14 & 0.16 & 0.21 & 0.22 & 0.25 \\
\cline{2-8} \noalign{\vspace{0.25ex}}
 & \multirow[t]{3}{*}{Safe RLHF} & DPO & 0.05 & 0.05 & 0.05 & 0.05 & 0.05 \\
 &  & PPO & 0.18 & 0.18 & 0.18 & 0.10 & 0.12 \\
 &  & SFT & 0.22 & 0.21 & 0.22 & 0.22 & 0.22 \\
\bottomrule
\end{tabular}
\caption{
   \label{tab:defence-mixture}
   Analysis of our strongest defences under varying flipping ratios.
}
\end{table*}

\section{Analysis of Refusal Loss and Lisa}
\label{app:refusal-lisa-hyperparameter}

As we saw in the main text, Refusal Loss and Lisa were the most effective online defences. However using these defences come at a significant cost: the hyperparamter controlling incorporation of their objective function must be very high. For both methods we needed to set this parameter to 100 which means that this part of the loss function is weighted 100x more than the original loss function. While we did not observe this to have an effect in slowing down learning in the TL;DR summarization task, down-weighting the original loss function means that learning that loss function is much more difficult. In Table~\ref{tab:hyperparam-variation}, we observe that on the SFT attacks from the main text, these defences are not effective at lower learning rates. The advantage of offline defences then is that the original loss function is not modified during harmless training as long as that offline defence satisfies the trainability condition from \cite{rosati2024immunizationharmfulfinetuningattacks}.

\begin{table}[h!]
\centering
\begin{tabular}{llcccccc}
\toprule
 & & 1 & 2 & 5 & 10 & 25 & 50 \\
\midrule
\multirow{2}{*}{Lisa} & 1k & 0.73 & 0.60 & 0.62 & 0.58 & 0.45 & 0.27 \\
 & 10k & 0.74 & 0.72 & 0.71 & 0.66 & 0.59 & 0.61 \\
\cline{1-8} \noalign{\vspace{0.25ex}}
\multirow{2}{*}{Refusal} & 1k & 0.70 & 0.67 & 0.33 & 0.15 & -- & -- \\
 & 10k & -- & -- & -- & 0.53 & 0.31 & 0.16 \\
\bottomrule
\end{tabular}
\caption{
    \label{tab:hyperparam-variation}
    Varying the $\rho$ and $\alpha$ parameters of Lisa and Refusal Loss at 1k and 10k attack sample sizes from BeaverTails illustrating the necessity of setting a very large hyperparameter. Note that we do not compute some settings for Refusal as it is not necessary to show that the attack was successful at that attack size and hyperparameter settings.
}
\end{table}

\section{Defence Against Label Noise}
\label{app:defence-mixture}

In this appendix, we present our results that were discussed in the main text regarding our defences and the observed effect in Table~\ref{tab:defence-mixture} that Lisa, RepNoise, and Refusal Loss are all quite effective protections in the harmful and harmless mixed attack setting at varying mixing ratios.

\section{Analysis of Reference Model Divergence}
\label{app:kl-constraints}
In Table~\ref{tab:analysis-of-kl}, we perform the same attack as Table~\ref{tab:vulnerability-method} using PPO with varying KL constraints and we find that for PPO increasing KL constraints do provide a naive defence simply by increasing the KL constraint. However, increasing the KL constraint makes leaning more difficult especially for tasks that are far from the original policy distribution. It is important to note that a consistent finding of the paper (Table~\ref{tab:online-defences}, Table~\ref{tab:offline-defences}) that successful defences often had very large KL divergence with the reference policies. Lisa was the only defence that consistently kept the KL divergence lower than the average KL divergence of a successful attack (which ranges from KL 100 to KL 200). We must however emphasize that KL is not a proper distance metric and future studies should follow up using proper distributional measures.

\begin{table}[h!]
\centering
\begin{tabular}{lrrrr}
\toprule
\textbf{KL} & \textbf{0.001} & \textbf{0.01} & \textbf{0.1} & \textbf{1} \\
\midrule
harmfulness & 0.71 & 0.73 & 0.53 & 0.14 \\
\bottomrule
\end{tabular}
\caption{
    \label{tab:analysis-of-kl}
    Analysis of KL divergence constraints on PPO-based RLHF and its impact on harmfulness scores.
}
\end{table}

Unfortunately, we did not find the same thing for varying $\beta$ parameters for DPO in Table~\ref{tab:analysis-of-beta}. However this may indicate sensitivity to properly set hyper-parameters rather than a core findings about the use of reference model policy constraints as a naive defence.
 
\begin{table}[h!]
\centering
\begin{tabular}{lrrrr}
\toprule
\textbf{$\beta$} & \textbf{0.001} & \textbf{0.01} & \textbf{0.1} & \textbf{1} \\
\midrule
harmfulness & 0.00 & 0.13 & 0.70 & 0.15 \\
\bottomrule
\end{tabular}
\caption{
    \label{tab:analysis-of-beta}
    Analysis of $\beta$ constraints on DPO-based RLHF and its impact on harmfulness scores.
}
\end{table}

We generally do not consider raising KL constraints in the main text as a defence baseline because of the potential downstream impacts on harmless learning as well as the disparate effect on DPO and PPO. Future work could consider an additional KL constraint with a specialized safety-trained model as a defence, however like with the use of reference policies in general this introduces an additional computational burden. Finally, we emphasize the finding in both Table~\ref{tab:online-defences} and Table~\ref{tab:offline-defences} that effective defences have large KL divergence from the original safety guarded reference policy.

\end{document}